\documentclass[accepted]{uai2025} % after acceptance, for a revised version; 
\usepackage[american]{babel}
\usepackage{natbib} % has a nice set of citation styles and commands
\usepackage{mathtools} % amsmath with fixes and additions
\usepackage{booktabs} % commands to create good-looking tables
\usepackage{tikz} % nice language for creating drawings and diagrams

\usepackage{amsmath,amsfonts,bm}
\usepackage{xcolor}
\usepackage{amssymb}
\usepackage{mathtools}
\usepackage{amsthm}
\usepackage{colortbl}

\def\eqref#1{equation~\ref{#1}}
\def\1{\bm{1}}

\newcommand{\tcen}[1]{\multicolumn{1}{c}{#1}}

\definecolor{royalblue}{rgb}{0.06, 0.75, 0.99}

\newcommand{\norm}[1]{\left\Vert#1\right\Vert}

\usepackage{xspace}
\newcommand*{\eg}{{\it e.g.}\@\xspace}
\newcommand*{\ie}{{\it i.e.}\@\xspace}

 \newtheorem{theorem}{Theorem}
 \newtheorem{lemma}{Lemma}
 
 \newtheorem{proposition}{Proposition}
 
\definecolor{mygray}{gray}{0.95}

\newcommand{\greybox}[1]{%
\vspace{-1.4em}
\begin{center}			% Centering minipage
\vspace{-0.5em}
\colorbox{mygray} {		% Set's the color of minipage
\begin{minipage}{0.987\linewidth} 	% Starts minipage
\centering
\vspace{-0.8em}
{#1}
\end{minipage}}			% End minipage
\end{center}
\vspace{-0.4em}%
}

\def\eqref#1{\ref{#1}}
\def\1{\bm{1}}

\DeclareMathAlphabet{\mathsfit}{\encodingdefault}{\sfdefault}{m}{sl}
\SetMathAlphabet{\mathsfit}{bold}{\encodingdefault}{\sfdefault}{bx}{n}

\newcommand{\E}{\mathbb{E}}

\newcommand{\R}{\mathbb{R}}

\newcommand{\sigmoid}{\sigma}

\usepackage[utf8]{inputenc} % allow utf-8 input
\usepackage[T1]{fontenc}    % use 8-bit T1 fonts
\usepackage{url}            % simple URL typesetting
\usepackage{booktabs}       % professional-quality tables
\usepackage{amsfonts}       % blackboard math symbols
\usepackage{nicefrac}       % compact symbols for 1/2, etc.
\usepackage{microtype}      % microtypography
\usepackage{xcolor}
\definecolor{linkcolor}{RGB}{74, 102, 146}

\usepackage{lipsum} 

\newcommand{\at}{a_t^\theta}
\newcommand{\bt}{b_t^\theta}

\newcommand{\ft}{f_t^\theta}

\newcommand{\diff}{g_t}

\newcommand{\methodname}{NCL++\@\xspace}

\usepackage{url}

\usepackage{cleveref}

\title{Simulation-Free Differential Dynamics \\through Neural Conservation Laws}

\author[1,2]{Mengjian Hua}
\author[2]{Eric Vanden-Eijnden}
\author[3]{Ricky T.Q. Chen}
\affil[1]{%
NYU Shanghai
}
\affil[2]{%
Courant Institute of Mathematical Sciences, New York University
}
\affil[3]{%
FAIR at Meta
  }
  
\begin{document}
\maketitle

\begin{abstract}
We present a novel simulation-free framework for training continuous-time diffusion processes over very general objective functions. 
Existing methods typically involve either prescribing the optimal diffusion process---which only works for heavily restricted problem formulations---or require expensive simulation to numerically obtain the time-dependent densities and sample from the diffusion process.
In contrast, we propose a coupled parameterization which jointly models a time-dependent density function, or probability path, and the dynamics of a diffusion process that generates this probability path.
To accomplish this, our approach directly bakes in the Fokker-Planck equation and density function requirements as hard constraints, by extending and greatly simplifying the construction of Neural Conservation Laws.
This enables simulation-free training for a large variety of problem formulations, from data-driven objectives as in generative modeling and dynamical optimal transport, to optimality-based objectives as in stochastic optimal control, with straightforward extensions to mean-field objectives due to the ease of accessing exact density functions. We validate our method in a diverse range of application domains from modeling spatio-temporal events to learning optimal dynamics from population data.
\end{abstract}

\section{Introduction}

Diffusion models have been widely adopted due to their ease of use and competitive performance in generative modeling \citet{ho2020denoisingdiffusionprobabilisticmodels, ma2024sitexploringflowdiffusionbased, chen2024flowmatchinggeneralgeometries}, by learning a diffusion process that interpolates between a data distribution and a Gaussian noise distribution \citet{song2021scorebasedgenerativemodelingstochastic, albergo2023stochasticinterpolantsunifyingframework, lipman2023flowmatchinggenerativemodeling}. 
However, their construction is heavily restrictive and only results in a simulation-free training algorithm for this simplest case. 
Recent works have adapted these ideas to train diffusion processes over more general objective functions, such as solving optimal transport or generalized Schr\"odinger bridge problems, but these methods all require simulating from the learned diffusion process to some varying degrees, and are generally more restrictive than simulation-based training approaches \citet{liu2024generalizedschrodingerbridgematching}.

We consider training diffusion processes over general objective functions%
\footnote{We denote $\partial_t \rho_t= \tfrac{\partial \rho_t}{\partial t}$, $\nabla \cdot (u_t\rho_t) = \sum_{d=1}^{D} \tfrac{\partial(u_t\rho_t)}{\partial x_d}$, and  $\Delta \rho_t= \sum_{d=1}^{D} \tfrac{\partial^2 \rho_t}{\partial x^2_d}$.}
\begin{align}
    \label{eq:problem_formulation}
    \min_{\rho, u} \;\;& \int_0^1 L(\rho_t, u_t) dt + F(\rho_0,\rho_1) \\
    \label{eq:fp_eq}
    \text{ s.t. }  \;& \partial_t \rho_t = - \nabla \cdot (u_t \rho_t) + \tfrac12\diff^2 \Delta \rho_t \\
    \label{eq:density_constraints}
    \;& \rho_t \geq 0,\quad  \int_{\R^D} \rho_t(x) dx = 1 \quad \forall t\in [0,1]
\end{align}
where $u_t(x): \R^{1+D} \rightarrow \R^D$ and $\rho_t(x): \R^{1+D} \rightarrow \R^+$ are the time-dependent velocity field and probability density function to be learned, and $g_t$ is a state-independent volatility that is given as part of the problem.
The functionals $L$ and $F$ can be quite general, including cases such as generative modeling from data observations, Schr\"odinger bridge problems, and mean-field control---we provide concrete examples in \Cref{sec:experiments}.
The constraints in \cref{eq:density_constraints} ensure the density function is properly normalized, while the constraint in \cref{eq:fp_eq}---the Fokker-Planck equation---implies that the diffusion process modeled by the stochastic differential equation (SDE)
\begin{equation}\label{eq:sde}
    d X_t = u_t(X_t) dt + \diff dW_t
\end{equation}
which transports particles in accordance with the marginal densities, \ie $X_t \sim \rho_t$.

Typical approaches will only directly parameterize $u_t$, the time-evolution of the particles, whereas $\rho_t$, the time-evolution of probability density function, is either unobtainable or only estimated through expensive numerical procedures~\citep{chen2019neuralordinarydifferentialequations,Kobyzev_2021}. As such, in order to sample from $\rho_t$, typically one transports particles starting from the initial time to time $t$, \textit{simulating} the diffusion process in \cref{eq:sde}.

In this work, we propose a novel parameterization of diffusion processes where we parameterize not only the dynamics $u_t$ but also the density $\rho_t$ in an explicit form, then we directly impose the Fokker-Planck equation (\ref{eq:fp_eq}) as a hard constraint on the model in order to couple these two quantities. 
To do so, we build upon ideas from Neural Conservation Laws (NCL; \citet{richter2022neural}) for imposing the continuity equation. 
We propose a reformulation of the NCL framework and significantly improve upon its prior construction; unlike prior work, we additionally include the \emph{density constraints} (\ref{eq:density_constraints}) into the model, enabling maximum likelihood training. 
We also find that the na\"ive construction introduces what we call a \emph{spurious flux phenomenon} which renders the velocity field unusable. We propose removing this phenomenon through the introduction of a carefully designed divergence-free component into the dynamics model that leaves the density invariant. In summary, our work introduces the following \textbf{main contributions}:
\begin{itemize}[leftmargin=0.15in,noitemsep,topsep=0pt]
    \item We give an improved analysis of the Neural Conservation Laws construction, generalizing it to diffusion processes and additionally imposing the density constraints (\ref{eq:density_constraints}). Compared to the original formulation, we can now train with the maximum likelihood objective.
    \item We discuss how the na\"ive construction leads to a spurious flux phenomenon, where the flux and velocity field do not vanish even as $x$ diverges. 
    We mitigate this problem by introducing carefully chosen divergence-free components to the flux that leaves the density invariant.
    \item We show that our method achieves state-of-the-art on a variety of spatio-temporal generative modeling data sets and on learning transport maps in cellular dynamics.
\end{itemize}
To the best of our knowledge, our method is the first to allow the training of a diffusion process with general objective functions---such as regularizing towards optimal transport, or with additional state costs, including mean-field cost functions---completely simulation-free, whereas existing methods require varying degrees of simulation.

\section{Related Work}

Markov processes described by ordinary and stochastic differential equations have been used across many application domains \citep{rubanova2019latent,karniadakis2021physics,cuomo2022scientific,wang2023scientific}, with the most general problem settings involving simulation-based methods. This refers to training neural differential equations of various kinds by simulating their trajectories and differentiating through the objective function. While some works have solved the memory issue with dfferentating through simulations \citep{chen2020neural,li2020scalable,chen2020learning}, it remains problematic to apply these at scale due to the computational cost of simulation. 
Furthermore, many probabilistic modeling applications \citep{grathwohl2018ffjord,chen2019residual,koshizuka2022neural} require the computation of the likelihood for maximum likelihood training, which can be more expensive than simulating trajectories.

This is where Neural Conservation Laws (NCL; \citet{richter2022neural}) come in, which is a modeling paradigm where the law of conservation such as \cref{eq:fp_eq} is directly enforced as a hard constraint. This allows optimization problems like~\cref{eq:problem_formulation} to be mapped on an unconstrained problem in the parameter space of an NCL model.
However, while the original NCL model \citep{richter2022neural} was able to bake in the constraint in \cref{eq:fp_eq}, they did not provide a scalable way to incorporate the density constraints in \cref{eq:density_constraints} which is key for enabling maximum likelihood training. 

An alternative framework is that of neural flows \citep{bilovs2021neural} which parameterize a flow model using a time-dependent normalizing flow \citep{papamakarios2021normalizing} instead of the velocity field. 
This approach avoids numerical simulation of ODEs but constraints such as invertibility must be enforced on the neural architecture and therefore limits its expressiveness. 

There are highly-scalable frameworks of diffusion models \citep{ho2020denoisingdiffusionprobabilisticmodels,song2021scorebasedgenerativemodelingstochastic}, inlcuding Flow Matching \citep{lipman2023flowmatchinggenerativemodeling}, and stochastic interpolants \citep{albergo2022building,albergo2023stochasticinterpolantsunifyingframework}. However, these methods can only solve a restricted set of problems, where samples from the optimal $\rho_0$ and $\rho_1$ are provided for training. They cannot handle the general problem setup of \cref{eq:problem_formulation} but instead directly prescribe the optimal solution which is then learned by a regression problem. 

% There have been works on generalizing this framework to handle more general problem settings \citep{liu2023generalized,neklyudov2023computational}, but these algorithms require simulation of the model to various degrees.

\section{Method} 

We describe a novel framework which directly parameterizes both a velocity field $u_t$ and a density $\rho_t$ that always satisfies the Fokker-Planck constraint in \cref{eq:fp_eq} and density constraints in \cref{eq:density_constraints}. Our method is built on top of ideas introduced in Neural Conservation Laws (NCL; \citet{richter2022neural}) through the use of differential forms, but we take an alternative construction while providing step-by-step derivations.
We then discuss how likelihood-based generative models can fit within our framework. The na\"ive construction, however, leads to a problem we call the \textit{spurious flux phenomenon} (\Cref{sec:spurious_flux}) which we resolve by introducing a divergence-free component (\Cref{sec:bt_cancellation}).

\subsection{Neural Conservation Laws}\label{sec:ncl}

In order to satisfy the Fokker-Planck constraint in \cref{eq:fp_eq}, we make use of a coupled parameterization of both a \emph{probability path} $\rho_t$, \ie a time-dependent density function, and a \emph{flux} $j_t(x) : \R^{D+1}\to\R^D$ that is designed, by construction, to always satisfy the continuity equation,
\begin{equation}\label{eq:continuity_eq}
    \partial_t \rho_t + \nabla \cdot j_t =0.
\end{equation}
This equation imposes the condition that the total probability in a system must be conserved, and that instantaneous changes in the probability can only be attributed to the local movement of particles following a continuous flow characterized together by $j_t$ and $\rho_t$. 
%The amount of flow in an infinitesimal area is captured by $j_t$.

We directly impose the continuity equation into the model as a hard constraint. This idea was previously explored in Neural Conservation Laws (NCL; \citet{richter2022neural}); however, its reliance on differential forms makes it difficult to extend, and they were not able to satisfy the density constraints in \cref{eq:density_constraints}. Instead, we propose a simplified alternative construction and derive the core building blocks of NCL that are necessary for our approach following only basic principles.

To model \cref{eq:continuity_eq}, we introduce two vector fields $\at(x):\R^{1+D} \rightarrow \R^D$ and $\bt(x):\R^{1+D} \rightarrow \R^D$ with parameters $\theta$, and set
\greybox{
\begin{align}
    \label{eq:rho_parameterization}
    \rho_t &= \nabla \cdot \at, \\
    \label{eq:flux_parameterization}
    j_t &= -\partial_t \at + \bt.%, \qquad \nabla \cdot \bt = 0.
\end{align}}
\vspace{-0.5em}
With this choice we have:
\begin{lemma}
\label{lemma:continuity}
    Let $\rho_t$ and $j_t$ be given by \cref{eq:rho_parameterization} and \cref{eq:flux_parameterization}, respectively. Then the continuity \cref{eq:continuity_eq} holds \text{iff} $b_t$ is divergence-free, i.e. $\nabla \cdot \bt = 0$.
\end{lemma}
\begin{proof}
    Plugging \cref{eq:rho_parameterization} and \cref{eq:flux_parameterization} into the left hand side of \cref{eq:continuity_eq},
\begin{equation}
    \partial_t \rho_t + \nabla \cdot j_t = \partial_t \nabla \cdot a_t^\theta - \nabla \cdot (\partial_t \at + \bt) = - \nabla \cdot b^\theta_t.
\end{equation}
Therefore \cref{eq:continuity_eq} holds \text{iff} $\nabla \cdot \bt = 0$, which verifies the claim.
\end{proof}
% Notice that while the density $\rho_t$ depends only on $\at$, the flux $j_t$ has an extra degree of freedom in $\bt(x):\R^{1+D} \rightarrow \R^D$, which can be any \emph{divergence-free} vector field, as shown by the following result:
Notice that $\rho_t$ depends only on $a^\theta_t$, while $j_t$ is affected by both $a_t^\theta$ and $b^\theta_t$. The extra degrees of freedom coming from~$b_t^\theta$ will be important in order to resolve what we call the \emph{spurious flux phenomenon} in \Cref{sec:spurious_flux}, and furthermore, it provides the needed flexibility in order to learn optimal solutions of~$u_t$ while leaving~$\rho_t$ invariant.

\subsection{Conversion to differential dynamics} \label{sec:sde_conversion}

In order to obtain the dynamics directly, we need to convert the continuity equation into the Fokker-Planck equation. Fortunately, the density and flux provide sufficient information in order to perform this conversion.
Any flux $j_t$ that satisfies the continuity equation in \cref{eq:continuity_eq} can be converted to a $u_t$ that satisfies \cref{eq:fp_eq} using the following identity:
\vspace{-0.9em}
\begin{center}			% Centering minipage
\vspace{-0.5em}
\colorbox{mygray} {		% Set's the color of minipage
\begin{minipage}{0.987\linewidth} 	% Starts minipage
\centering
\vspace{-0.8em}   
\begin{equation}\label{eq:drift_from_flux}
    u_t = \frac{j_t}{\rho_t} + \tfrac12\diff^2\nabla \log \rho_t.
\end{equation}
\end{minipage}}			% End minipage
\end{center}
\vspace{-0.5em}
This can be verified by plugging \cref{eq:drift_from_flux} into \cref{eq:fp_eq}.
\begin{equation}
\begin{split}
    \partial_t \rho_t& =-\nabla \cdot \left(\rho_t \left( \frac{j_t}{\rho_t} + \tfrac12\diff^2\nabla \log \rho_t \right)\right) + \tfrac12\diff^2 \Delta \rho_t(x)\\
    &= - \nabla \cdot \left( j_t + \tfrac12\diff^2\nabla \rho_t \right) + \tfrac12\diff^2 \Delta \rho_t(x)
    = - \nabla \cdot j_t 
\end{split}
\end{equation}
Thus, by parameterizing a single vector field $\at$, we can model both a density $\rho_t$ and a velocity field $u_t$ that satisfies the constraint in \cref{eq:fp_eq}. This allows us to turn the constrained optimization problem in \cref{eq:problem_formulation} into an unconstrained optimization in the parameters $\theta$. Furthermore, as we are given direct access to $\rho_t$, we do not need to solve the Fokker-Planck equation (\cref{eq:fp_eq}) from $u_t$, typically requiring an extremely expensive procedure. This enables a new paradigm of simulation-free methods for training diffusion models over general objective functions.

\subsection{Designing $\at$ through likelihood-based models}

In order to model valid probability density functions, we must also satisfy the density constraints in \cref{eq:density_constraints}. 
In addition, we wish to design our choice of $\at$ such that: (i) $\rho_t$ can be exactly sampled from at any time value $t$, (ii) computation of $\rho_t$ incurs minimal computational cost, and (iii) the model is flexible enough for practical applications. We will show that autoregressive likelihood-based models nicely fit within our framework and satisfies all of the above desirables.

Consider a time-dependent \emph{autoregressive probabilistic model} which decomposes the joint distribution over all $D$ variables given the natural ordering, 
% \ricky{remove $\pi$ as a subscript and impose natural ordering on notation}
\begin{equation}
\label{eq:auto:reg}\textstyle
    \rho_t(x) = \prod_{i=1}^D f^\theta_t(x_{i} | x_{1:i-1}),
\end{equation}
and denote by $F^\theta_t(x_{i} | x_{1:i-1}) = \int_{-\infty}^{x_{i}} f^\theta_t(y | x_{1:i-1}) dy$ the associated cumulative probability distributions (CDF).
Letting $[\at]_{i}$ denote the $i$-th coordinate of $\at$,  this model can be constructed by setting 
\vspace{-0.9em}
\begin{center}			% Centering minipage
\vspace{-0.5em}
\colorbox{mygray} {		% Set's the color of minipage
\begin{minipage}{0.987\linewidth} 	% Starts minipage
\centering
\vspace{-0.8em}   
\begin{equation}\label{eq:autoregressive_at}
     \begin{aligned}
    [\at]_i(x) &=0, \quad\quad\qquad \text{if } \ i \not= D\\
    [\at]_D(x) &=
    F^\theta_t (x_{D}| x_{1:D-1}) \prod_{j <D} f^\theta_t(x_{j}|x_{1:j-1}),
 \end{aligned}
\end{equation}
\end{minipage}}			% End minipage
\end{center}
\vspace{-0.5em}
This gives $\rho_t(x) = \nabla \cdot \at(x) = \prod_{i=1}^{D} f^\theta_t(x_i| x_{1:i-1})$, recovering eq.~(\ref{eq:auto:reg}) 

Alternatively, we can consider a factorized model, where $f$ is a parameterized probability density function (PDF) and $x_i$ does not depend on other variables:
\begin{align}\label{eq:factorized_model}
    \rho_t(x) = \prod_{i=1}^{D}  f^\theta_t (x_i),
\end{align}
which we will refer to as the {\it factorized model}. 

\paragraph{Choice of $F^\theta_t$ as mixture of logistics.} While one may directly parameterize the functions $F^\theta_t$ using monotonic neural networks \citep{sill1997monotonic,daniels2010monotone}, resulting in a universal density approximator, we decide to use a simpler construction using mixture of logistic distributions. 
Mixture of logistics has been a common choice among likelihood-based generative modeling frameworks, from normalizing flows \citep{kingma2016improved,ho2019flowpp} to autoregressive models \citep{salimans2017pixelcnnpp}. Similarly, mixture of logistics is sufficient flexible for our use cases, as we only need to model per-coordinate conditional distributions. For our autoregressive model, we use a mixture of logistics to describes the CDF as
\begin{equation}\label{eq:autoregressive_mixture}
\begin{aligned}
    \textstyle
	&F^\theta_t(x_{i} | x_{1:i-1}) = \sum_{l=1}^L \alpha^{\theta}_{l}(x_{1:i-1},t)  \sigmoid\left( z^\theta_l(x_{i},x_{1:i-1},t ) \right),
\end{aligned}
\end{equation}
where $\sigmoid(x) = 1/(1+ \exp(-x))$ is the sigmoid function and we defined
\begin{equation}
    z^\theta_l(x_{i},x_{1:i-1},t ) = s^{\theta}_{l}(x_{1:i-1},t) \left(x_{i} - \mu^{\theta}_{l}(x_{1:i-1},t)\right).
\end{equation}
Here, $\mu^{\theta}_{l}(x_{1:i-1},t)$ and $s^{\theta}_{l}(x_{1:i-1},t)$ correspond to the mean and inverse scale of a logistic distribution, respectively, while $\alpha^{\theta}_{l}(x_{1:i-1},t)$ are mixture weights. All functions are parameterized using autoregressive neural networks. These correspond to probability density functions
\begin{equation}
\begin{aligned}
    \textstyle
	& f^\theta_t(x_{i}| x_{1:i-1}) =   \sum_{l=1}^L \alpha^{\theta}_{l}(x_{1:i-1},t) \left[
	s^{\theta}_{l}(x_{1:i-1},t) \right. \\ &\quad \times \left. \sigmoid \left( z^\theta_l (x_{i},x_{1:i-1},t )\right) \sigmoid \left( -z_l (x_{i},x_{1:i-1},t )\right) 
	\right],
\end{aligned}
\end{equation}
As for the factorized model, since we remove all the dependecies of the CDF on the prior coordiates, we can define $F^\theta_t(x_{i})$ via 
\begin{equation}\textstyle
	F^\theta_t(x_{i}) = \sum_{l=1}^L \alpha^{\theta}_{l}(t) \left[ \sigmoid\left( s^{\theta}_{l}(t)  \left(x_{i} - \mu^{\theta}_{l}(t)\right) \right) \right],
\end{equation}
where the mean, inverse scale, and the mixture weights are functions depend on $t$ only. 

While these constructions for the factorized model and the autoregressive lead to a proper density, the keen reader may notice that the flux constructed from \cref{eq:flux_parameterization} using this $\at$ is problematic as the flux will be exactly zero in all but one coordinate. This is not the core of the problem but rather a manifestation of the spurious flux phenonmenon, which we will describe in \Cref{sec:spurious_flux}. We will later go in depth on how to construct a proper flux by making use of the extra degree of freedom we have in designing $\bt$ later in \Cref{sec:bt_cancellation}. 

\subsection{The spurious flux phenomenon}\label{sec:spurious_flux}

The choice of $\at$ above guarantees that $\rho_t$ is positive and normalizes to exactly one. However, using only this $\at$ and setting $\bt=0$ in \cref{eq:flux_parameterization}  turns out to be problematic for the flux $j_t$. Indeed, given any box-shaped region $\mathcal{X} = [-L,L]^D$ with boundary denoted $\partial \mathcal{X}$ and normal vector $\hat n(x)$, we can use the divergence theorem to obtain
\begin{equation}
\begin{aligned}
    \int_{\mathcal{X}} \rho_t(x) dx &= \int_{\mathcal{X}}  \nabla \cdot \at(x) dx\\
    &= \int_{\partial \mathcal{X}} \hat n(x)\cdot a_t^\theta(x) dS(x) > 0.
\end{aligned}
\end{equation}
where $dS(x)$ is the surface measure on $\partial \mathcal{X}$. 
This quantity is nonzero no matter how large $L$ is, and approaches one as $L \rightarrow \infty$ since $\int_{\R^D} \rho_t(x) dx =1$.

This implies that $\at$ is necessarily nonzero somewhere even outside of the support of $\rho_t$.
Therefore, if we set $b^\theta_t=0$ so that $j_t = - \partial_t \at$ by \cref{eq:flux_parameterization}, since $\at$ is not constant in $t$ in general, the flux does not decay to zero even outside the support of $\rho_t$, even though $\rho_t$ goes to zero. 
This is problematic for two reasons: (i) we take $j_t/\rho_t$ in order to construct $u_t$, which will diverge, and (ii) because the divergence theorem holds for \textit{any} region, this introduces unwanted behavior even at finite $x$, as can be seen in Figure~\ref{fig:flux_illustration}. 

\begin{figure*}[t]
    \centering
\includegraphics[width=0.9\linewidth]{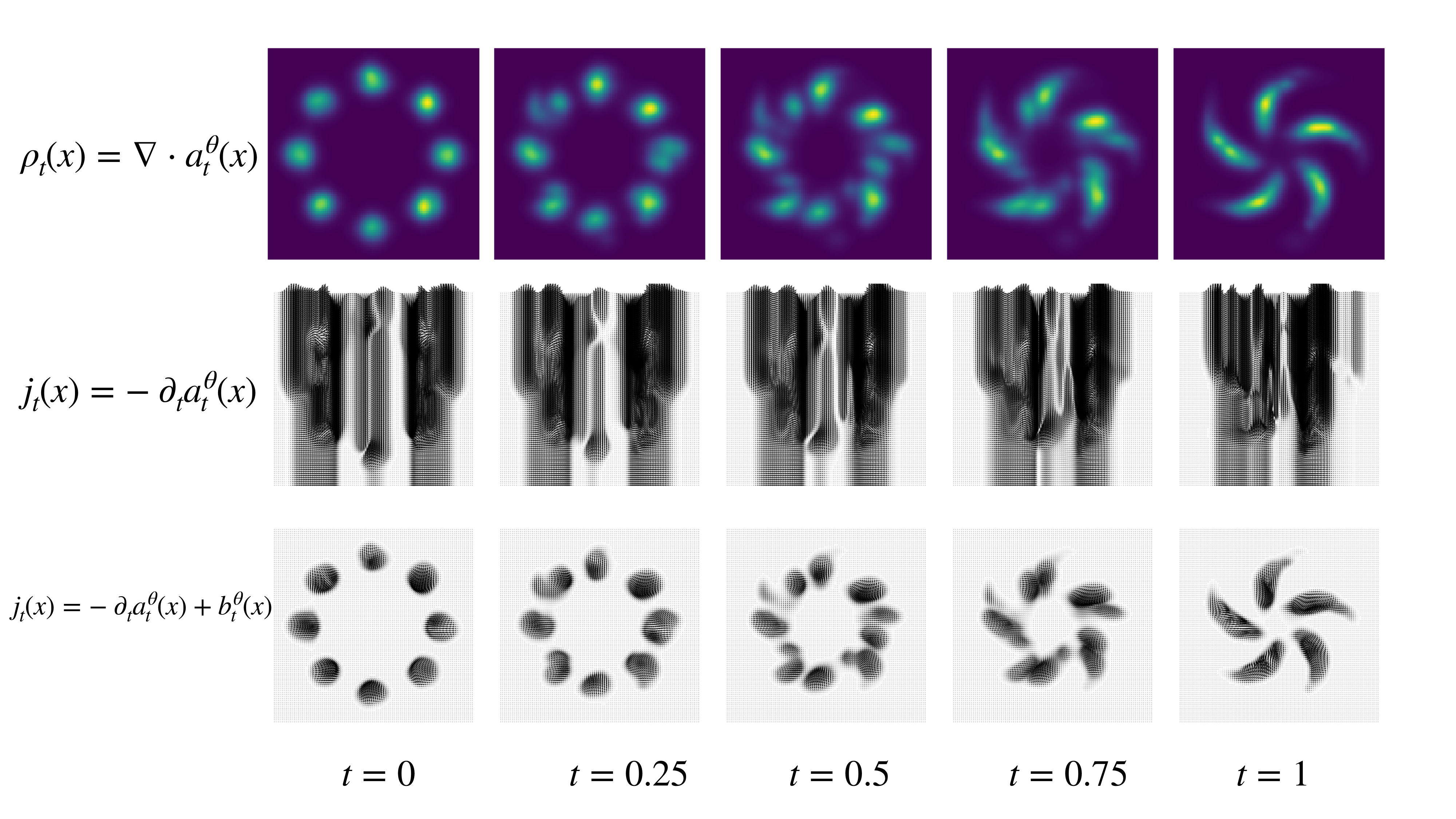}
    \caption{Illustration of the \textit{spurious flux phenomenon} and its removal with a divergence-free vector field $\bt$. 
    (\textit{top}) The trained marginal distributions in 2D. 
    (\textit{middle}) The flux field $j_t = -\partial_t \at$ without any flux cancellations, where we see there are spurious fluxes. 
    (\textit{bottom}) The flux field $j_t = -\partial_t \at + \bt$ with $\bt$ defined in \Cref{sec:bt_cancellation}, and we now see that the flux field vanishes properly. }
    \label{fig:flux_illustration}
\end{figure*}

This \textit{spurious flux phenomenon}, where the flux is nonzero even as $x$ diverges, exists generally for any construction that does not enforce $\lim_{x\to\infty} \partial_t \at=0$, and it can be easily formalized in the case of the autoregressive construction of $\at$  in \cref{eq:autoregressive_at}, as summarized by the following result:
\begin{lemma}
Let $\rho_t(x) = \nabla \cdot \at(x) = \prod_{i=1}^D f^\theta_t(x_{i})$ and $j_t(x) = -\partial_t \at(x)$ with $a^\theta_t(x)$ given by \cref{eq:autoregressive_at}. Then $\lim_{x_{D} \to + \infty} | j_t (x) |^2 \not = 0$.
\end{lemma}

\begin{proof}
    If $j_t(x) = -\partial_t \at(x)$ with $a^\theta_t(x)$ given by \cref{eq:autoregressive_at}, the $D$-th coordinate of the flux is 
\begin{equation*}
\begin{aligned}
    \textstyle
    [j_t]_{D}(x) 
    = & - \partial_t F^\theta_t (x_{D} | x_{1:D-1}) \prod_{i=1}^{D-1} f_t^\theta(x_{i} | x_{1:i-1}) \\ &  - F^\theta_t (x_{D} | x_{1:D-1}) \partial_t \left( \prod_{i=1}^{D-1} f_t^\theta(x_{i} | x_{1:i-1}) \right).
\end{aligned}
\end{equation*}
Since  $\lim_{x_{D} \to \infty} F_t (D) =1$ and therefore $\lim_{x_{D} \to \infty}\partial_t F_t (D) =0$, we deduce
\begin{equation}
\label{eq:spurious_flux_autoregressive}
 %   \textstyle
    \lim_{x_{D} \rightarrow \infty} [j_t]_{D}(x) =  - \partial_t \left( \prod_{i=1}^{D-1} f_t^\theta(x_{i}| x_{1:i-1}) \right) \not=0
\end{equation}
and the claim of the lemma follows.
\end{proof} 
Next we show that we can zero the spurious flux by using the extra degree of freedom offered by the divergence-free field $\bt$ in the construction of \cref{eq:flux_parameterization}, as discussed next.

\subsection{Designing $\bt$ to combat the spurious flux phenomenon}\label{sec:bt_cancellation}

In order to remove the spurious flux in \cref{eq:spurious_flux_autoregressive}, we must cancel it out with a term that has the same limiting behavior. To this end, we propose adding the quantity $\sigmoid(x_D) \partial_t \left( \prod_{i=1}^{D-1} f^\theta_t(x_{i}| x_{1:i-1}) \right)$ to the $D$-th coordinate of the flux, where $\sigmoid$ is the sigmoid function, or more generally, a function/neural network approaches $1$ as its input goes to infinity. By adding this \textit{cancellation} term to the spurious flux from \cref{eq:spurious_flux_autoregressive}, we exactly remove the limiting behavior:
\begin{equation}
\begin{aligned}
     \lim_{x_{D} \rightarrow \infty} \textstyle & - \partial_t \left(F^\theta_t (x_{D}| x_{1:D-1}) \prod_{i=1}^{D-1} f_t^\theta(x_{i}| x_{1:i-1}) \right) \\ & + \sigmoid(x_{D}) \partial_t \left( \prod_{i=1}^{D-1} f^\theta_t(x_{i}| x_{1:i-1}) \right) = 0.
\end{aligned}
\end{equation}
However, we must construct $\bt$ to be divergence-free in order to leave $\rho_t$ invariant. Notice that this cancellation term has the form
\begin{equation}\textstyle
\label{eq:spurious_flux_design}
\begin{aligned}
    [\bt]_{D}(x) & = \sigmoid(x_{D}) \partial_t \left( \prod_{i=1}^{D-1} f^\theta_t(x_{i} | x_{1:i-1}) \right) \\ & = 
\frac{\partial}{\partial{x_{D-1}}} \Big[  \sigmoid(x_{D}) \partial_t \Big( F^\theta_t(x_{D-1} | x_{1:D-2})  \\ & \qquad \quad \times \prod_{i=1}^{D-2} f^\theta_t(x_{i} | x_{1:i-1}) \Big) \Big].
\end{aligned}
\end{equation}
To ensure $\bt$ is divergence-free, we add a \textit{compensating} term to the ${D-1}$-th coordinate,
\begin{equation}\textstyle
\label{eq:spurious_flux_from_divfree}
\begin{aligned}
    [\bt]_{{D-1}}(x) & = - \frac{\partial}{\partial x_{{D}}} \Big[  \sigmoid(x_{D}) \partial_t \Big( F^\theta_t(x_{D-1} | x_{1:D-2}) \\ & \qquad\quad \times\prod_{i=1}^{D-2} f^\theta_t(x_{i} | x_{1:i-1}) \Big) \Big].
\end{aligned}
\end{equation}
This results in a $\bt$ that is divergence-free since
\begin{equation}%\textstyle
	\nabla \cdot \bt = \frac{\partial}{\partial x_{{D}}}  [\bt]_{{D}} + \frac{\partial}{\partial x_{{D-1}}} [\bt]_{{D-1}} = 0.
\end{equation}
However, the $\bt$ in \cref{eq:spurious_flux_from_divfree} introduces a new spurious flux in the ${D-1}$ coordinate since $[\bt]_{{D-1}} \neq 0$ as $x_{{D-1}} \rightarrow \infty$. To completely remove spurious flux while keeping $\bt$ divergence-free, we must recursively add \textit{cancellation} and \textit{compensating} terms to each coordinate, until every coordinate has their spurious flux removed. This results in the following vector field for the general case:
\vspace{-0.9em}
\begin{center}			% Centering minipage
\vspace{-0.5em}
\colorbox{mygray} {		% Set's the color of minipage
\begin{minipage}{0.987\linewidth} 	% Starts minipage
\centering
\vspace{-0.8em}   
\begin{equation}\label{eq:autoregressive_bt}
[\bt]_{i}(x) =
\begin{cases}
    & \sigma(x_{i})\partial_t \prod_{j=1}^{D-1} f^\theta_t(x_{j}| x_{1:j-1}), \;\;\text{if } i = D\\
    & - \left( \prod_{j=2}^D \sigma'(x_{j}) \right) \partial_t  F^\theta_t(x_{1}),  \quad \text{if } i = 1\\
    & \left( \prod_{j=i+1}^D \sigma'(x_{{j}}) \right) \left(\sigma(x_{i}) -  F^\theta_t(x_{i} | x_{1:i-1})\right)  \\ & \times \partial_t \left( \prod_{j=1}^{i-1} f^\theta_t(x_{j} | x_{1:j-1})\right) \\ & -\left( \prod_{j=i+1}^D\sigma'(x_{{j}}) \right)  \partial_t  F^\theta_t(x_{i} | x_{1:i-1}) \\ &  \times\left( \prod_{j=1}^{i-1} f^\theta_t(x_{j} | x_{1:j-1} )\right), \quad  \text{otherwise}
\end{cases}
\end{equation}
\end{minipage}}			% End minipage
\end{center}
\vspace{-0.5em}

The following results show that the $\bt$ in \cref{eq:autoregressive_bt} is divergence-free and that it completely removes the spurious flux problem.
\begin{lemma}
\label{lemma:div_free}
    The vector field $\bt$ in \cref{eq:autoregressive_bt} is divergence-free, \ie $\nabla \cdot \bt = 0$. We provide proof in \Cref{app:proof_1}. 
\end{lemma}
\begin{theorem}
\label{theorem:no_flux}
Let $\rho_t$ and $j_t$ be given by \cref{eq:rho_parameterization} and \cref{eq:flux_parameterization}, respectively, with $a^t_\theta$ given by~\cref{eq:autoregressive_at} and $b_t^\theta$ by \cref{eq:autoregressive_bt}. Then the continuity \cref{eq:continuity_eq} holds, the density satisfies $\rho_t>0$ and $\int_{\R^D}\rho_t(x) dx=1$, and in addition there are no spurious flux, i.e. $j_t \rightarrow 0$ as $x \rightarrow \infty$. We provide proof in \Cref{app:proof_1}. 
\end{theorem}

Note that in our implementation, we compute all quantities in \cref{eq:autoregressive_bt} in parallel across all coordinates using autoregressive architectures, and in logarithm space for numerical stability. The derivatives $\partial_t$ are computed using memory-efficient forward-mode automatic differentiation, so the total cost of computing \cref{eq:autoregressive_bt} has the same asymptotic compute cost of a single evaluation of the autoregressive model.

\subsection{The factorized case: simplifications and generalizations}\label{sec:factorized_case}

The vector field in \cref{eq:autoregressive_bt} can be drastically simplified for the factorized case by setting $\sigma(x_i) = F^\theta_t (x_i)$, which gives 
\begin{equation}\label{eq:factorized_bt:prob}
[\bt]_{i}(x) =
\begin{cases}
    & F_t^\theta(x_{D}) \partial_t \left( \prod_{j=1}^{D-1} f^\theta_t(x_{j}) \right), \quad  \text{if } i = D\\
    & -\left( \prod_{j=2}^D f^\theta_t(x_{j}) \right) \partial_t  F^\theta_t(x_{1}) , \quad  \text{if } i = 1\\
    & -\left( \prod_{j=i+1}^Df_t^\theta(x_{{j}}) \right) \partial_t  F^\theta_t(x_{i}) \\ & \times\left(\prod_{j=1}^{i-1} f^\theta_t(x_{j} )\right), \quad \text{otherwise}.
\end{cases}
\end{equation}
Substituting this back into \cref{eq:drift_from_flux} results in the simplified velocity field (for $g_t =0$):
\vspace{-0.9em}
\begin{center}			% Centering minipage
\vspace{-0.5em}
\colorbox{mygray} {		% Set's the color of minipage
\begin{minipage}{0.987\linewidth} 	% Starts minipage
\centering
\vspace{-0.8em}  
\begin{equation}\label{eq:factorized_ut}
\begin{aligned}
    [u^\theta_t]_{i}(x) & = j_t^\theta (x)/\rho_t^\theta (x) \\ &  = (- \partial_t \at (x) + \bt(x)) / \rho_t^\theta (x)\\
    & = -\frac{\partial_t F_t^\theta (x_i) }{f_t^\theta (x_i)}
\end{aligned}
\end{equation}
\end{minipage}}			% End minipage
\end{center}
\vspace{-0.5em}
for all $i \in \{1,\ldots,D\}$, which is easy to implement and compute in practice. Furthermore, we note that for the factorized model, the velocity is always {\it kinetic optimal} as $u_t(x)$ in \cref{eq:factorized_ut} is a gradient field. In particular,  it means that it is the velocity that results in the shortest paths out of all velocities that generate this $\rho_t$.

To increase the flexibility of the factorized model, note that we can combine multiple pairs of $(\rho_t^k, u_t^k)$ into a mixture model with coefficients $\gamma^k$:
\begin{equation}\label{eq:mixture_rho}
    \rho_t(x) = \sum_{k=1}^K \gamma^k \rho_t^k(x), \quad u_t(x) = \sum_{k=1}^K \frac{\gamma^k \rho_t^k(x)}{\rho_t(x)} u_t^k(x).
\end{equation}
\begin{proposition}
\label{prop:mixture}
    If each pair of $\rho_t^k$ and $u_t^k$ satisfy the Fokker-Planck equation as in \cref{eq:fp_eq}, then the $\rho_t$ and $u_t$ as defined in \cref{eq:mixture_rho} also satisfy the Fokker-Planck equation. Proof is provided in \Cref{app:prop1_proof}. 
\end{proposition}

\subsection{Learning an independent divergence-free component}\label{sec:divfree}

While the choice of $\bt$ in \cref{eq:autoregressive_bt} removes the spurious nonzero flux values at infinity, this parameterization lacks the flexibility in optimizing $j_t$, \eg it does not necessarily correspond to kinetic optimal velocity fields $u_t$, unless we use the factorized model discussed in \Cref{sec:factorized_case}.
In order to handle a wider range of applications where we do optimize over $u_t$, we can include a flexible learnable component into $j_t$ that leaves the continuity equation invariant. 

Let the new flux field be parameterized as 
\begin{equation}
   j_t= -\partial_t \at + \bt + v_t^\theta, \qquad \text{ where } \nabla \cdot v_t^\theta = 0,
\end{equation}
so $\ft: \R^{D+1} \rightarrow \R^D$ is a divergence-free vector field. This construction still satisfies \cref{eq:continuity_eq} because 
\begin{equation}
    \partial_t \rho_t + \nabla \cdot j_t = \partial_t (\nabla \cdot \at) - \nabla \cdot (\partial_t \at + \bt + v_t^\theta) = 0 
\end{equation}

To satisfy the divergence-free constraint, we adopt the construction in \cite{richter2022neural} and parameterize an matrix-valued function $A_t^\theta: \R^{D+1} \rightarrow \R^{D\times D}$ with neural networks and we let 
\begin{equation}
    v_t^\theta = \nabla \cdot \left(A_t^\theta - (A_t^\theta)^T \right)
\end{equation}
where the divergence is taken over the rows of the anti-symmetric matrix $A_t^\theta - (A_t^\theta)^T$. Let $A_{t;i,j}^\theta$ denote the $(i,j)$ entry of $A_t^\theta$. We can easily verify that $v_t^\theta$ is divergence-free with the following:
\begin{equation}
\begin{aligned}
    \nabla \cdot v_t^\theta = & \sum_{i = 1}^D \sum_{j = 1}^D \partial_{x_i} \partial_{x_j} \left(A_{t;i,j}^\theta - (A_{t;i,j}^\theta)^T \right) \\ & -  \partial_{x_j} \partial_{x_i} \left(A_{t;i,j}^\theta - (A_{t;i,j}^\theta)^T \right)  = 0
\end{aligned}
\end{equation}

\section{Experiments}\label{sec:experiments}

In each of the following sections, we consider broader and broader problem statements, where each successive problem setting roughly builds on top of the previous ones.
Throughout, we parameterize the density $\rho_t$ and a flux $j_t$ following \Cref{sec:ncl} in order to satisfy the continuity equation, and compute the velocity field $u_t$ using \cref{eq:drift_from_flux}. 
All models are trained without simulating the differential equation in \cref{eq:sde}. 
While there exist simulation-free baselines for the first few settings (Sections \ref{sec:stpp} \& \ref{sec:ot}), to the best of our knowledge, we are the first truly simulation-free approach for the more complex problem setting involving mean-field optimal control (\Cref{sec:soc}). Experimental details are provided in \Cref{app:experimental_setup}.

\begin{table}
\centering
\setlength{\tabcolsep}{2.5pt}
\resizebox{\linewidth}{!}{%
% \small
\begin{tabular}{@{} l r r r r r@{}}\toprule
Model & \tcen{Pinwheel} & \tcen{Earthquakes JP}
& \tcen{COVID-19 NJ} & \tcen{CitiBike}\\
\cmidrule(r){1-1}\cmidrule(lr){2-2} \cmidrule(lr){3-3} 
\cmidrule(lr){4-4} \cmidrule(lr){5-5}
Conditional KDE (\cite{chen2020neural})& 
$2.958$ {\tiny $\pm 0.000$}& 
$2.259$ {\tiny $\pm 0.001$} & 
$2.583$ {\tiny $\pm 0.000$}& 
$2.856$ {\tiny $\pm 0.000$}\\
Neural Flow (\cite{bilovs2021neural})& 
\small \color{gray} N/A & 
$1.633 $& 
$1.916 $& 
$2.280 $\\
% %
CNF (\cite{chen2020neural})& 
$2.185$ {\tiny $\pm 0.003$}& 
$1.459$ {\tiny $\pm 0.016$}& 
$2.002$ {\tiny $\pm 0.002$}& 
$2.132$ {\tiny $\pm 0.012$}\\

\cmidrule(r){1-1}\cmidrule(lr){2-2} \cmidrule(lr){3-3} 
\cmidrule(lr){4-4} \cmidrule(lr){5-5}
\methodname (Factorized) & 
$2.028$ {\tiny $\pm 0.062$}& 
$1.217$ {\tiny $\pm 0.024$}& 
$1.846$ {\tiny $\pm 0.012$}& 
$1.462$ {\tiny $\pm 0.033$}\\
\methodname (Autoregressive) & 
$\mathbf{1.936}$ {\tiny $\pm 0.083$}&  
$\mathbf{1.184}$ {\tiny $\pm 0.031$}& 
$\mathbf{1.732}$ {\tiny $\pm 0.009$}& 
$\mathbf{1.239}$ {\tiny $\pm 0.024$}\\
\bottomrule
\end{tabular}
}
\caption{Negative log-likelihood per event on held-out test data (lower is better).}
\label{table:results_STPP}
\end{table}

\subsection{Spatio-temporal generative modeling}\label{sec:stpp}

Our goal is to fit the model to data observations from an unknown data distribution $q(t, x)$. 
We consider the unconditional case of generative modeling where samples are obtained from marginal distributions across time, while the individual trajectories are unavailable.
As a canonical choice, we use the cross entropy as the loss function for learning $\rho_t$.
\begin{equation}
    L_{\text{GM}} = \E_{t, x\sim q(t, x)} \left[ - \log \rho_t(x) \right]
\end{equation}

We consider datasets of spatial-temporal events preprocessed by \citet{chen2020neural} and these datasets are sampled randomly in continuous time. We take only the spatial component of these datasets, as this is our core contribution.
To evaluate the capability of our method on modeling spatial-temporal processes, we test our proposed method on these datasets and compare against state-of-the-art models on these datasets by \cite{chen2020neural} and \cite{bilovs2021neural}. 

We report the log-likelihoods per event on held-out test data of our method and baseline methods in Table~\ref{table:results_STPP}, highlighting that our method outperforms the baselines with substantially better spatial log-likelihoods across all datasets considered here. 

\subsection{Learning to transport with optimality conditions}\label{sec:ot}

\begin{table}
\centering
\setlength{\tabcolsep}{2.5pt}
\resizebox{\linewidth}{!}{%
% \small
\begin{tabular}{@{} l r r r r r@{}}\toprule
Model & \tcen{$W_2(q_{t_1}, \widehat{q}_{t_1})$} & \tcen{$W_2(q_{t_2}, \widehat{q}_{t_2})$}
& \tcen{$W_2(q_{t_3}, \widehat{q}_{t_3})$} & \tcen{$W_2(q_{t_4}, \widehat{q}_{t_4})$}\\
\cmidrule(r){1-1}\cmidrule(lr){2-2} \cmidrule(lr){3-3} 
\cmidrule(lr){4-4} \cmidrule(lr){5-5}
OT-flow & 
$0.75$ & 
$0.93$  & 
$0.93$ & 
$0.88$ \\
Entropic Action Matching & 
$0.58$ {\tiny $\pm 0.015$}& 
$0.77$ {\tiny $\pm 0.016$}& 
$\mathbf{0.72}$ {\tiny $\pm 0.007$}& 
$0.74 ${\tiny $\pm 0.017$}\\
% %
Neural SDE & 
$0.62$ {\tiny $\pm 0.016$}& 
$0.78$ {\tiny $\pm 0.021$}& 
$0.77$ {\tiny $\pm 0.017$}& 
$0.75$ {\tiny $\pm 0.017$}\\

\cmidrule(r){1-1}\cmidrule(lr){2-2} \cmidrule(lr){3-3} 
\cmidrule(lr){4-4} \cmidrule(lr){5-5}
\methodname (Factorized, Directly Sampled) & 
$0.56$ {\tiny $\pm 0.009$}& 
$0.79$ {\tiny $\pm 0.012$}& 
$0.74$ {\tiny $\pm 0.010$}& 
$0.72$ {\tiny $\pm 0.006$}\\
\methodname (Autoregressive, Directly Sampled) & 
$\mathbf{0.52}$ {\tiny $\pm 0.004$}& 
$\mathbf{0.74}$ {\tiny $\pm 0.005$}& 
$\mathbf{0.72}$ {\tiny $\pm 0.003$}& 
$\mathbf{0.69}$ {\tiny $\pm 0.004$}\\
\methodname (Factorized, Transported) & 
$0.58$ {\tiny $\pm 0.015$}& 
$0.80 $ {\tiny $\pm 0.007$}& 
$0.76$ {\tiny $\pm 0.009$}& 
$0.75$ {\tiny $\pm 0.009$}\\
\methodname (Autoregressive, Transported) & 
$0.53$ {\tiny $\pm 0.013$}& 
$0.76$ {\tiny $\pm 0.008$}& 
$0.73$ {\tiny $\pm 0.005$}& 
$0.71$ {\tiny $\pm 0.008$}\\
\bottomrule
\end{tabular}
}
\caption{The Wasserstein-2 distance between the test marginals and marginal distributions from the model calculated by the test samples and the samples obtained from directly sampling from the model. For our own methods, we report standard deviation estimated across 20 runs.}
\label{table:results_OT}
\end{table}

\begin{table}
\centering
\setlength{\tabcolsep}{2.5pt}
\resizebox{\linewidth}{!}{%
% \small
\begin{tabular}{@{} l r r r r r@{}}\toprule
Model  & \tcen{$W_2(q_{t_1}, \widehat{q}_{t_2})$}
& \tcen{$W_2(q_{t_2}, \widehat{q}_{t_3})$} & \tcen{$W_2(q_{t_3}, \widehat{q}_{t_4})$}\\
\cmidrule(r){1-1}\cmidrule(lr){2-2} \cmidrule(lr){3-3} 
\cmidrule(lr){4-4}

\methodname (Factorized) & 
$3.45$ {\tiny $\pm 0.125$}& 
$3.67$ {\tiny $\pm 0.103$}& 
$4.09$ {\tiny $\pm 0.147$}\\
\methodname (Autoregressive) & 
$\mathbf{2.85}$ {\tiny $\pm 0.075$}& 
$\mathbf{3.14}$ {\tiny $\pm 0.082$}& 
$\mathbf{3.62}$ {\tiny $\pm 0.097$}\\
\bottomrule
\end{tabular}
}
\caption{The Wasserstein-2 distance between the distributions transported from the test marginals at $t_i$ and the test marginals at $t_{i+1}$. We use this Wasserstein-2 distance to measure how kinetically optimal our trained maps are. We report the mean and standard deviation estimated across 20 runs. }
\label{table:results_kinetic_OT}

\end{table}

We next consider settings where the data are only sparse observed at select time values, and the goal is to learn a transport between each consecutive observed time values, subject to some optimality conditions. The simplest case is dynamic optimal transport \citet{villani2021topics}, where we introduce a kinetic energy to the loss function in order to recover short trajectories between consecutive time values.
\begin{equation}
\begin{aligned}
    L_{\text{OT}} = & \sum_{t \in \{t_i\}_{i=1}^n} \E_{x\sim q_{t_i}(x)} \left[ - \log \rho_t(x) \right] \\ & + \int_{t_0}^{t_n} \E_{x \sim \rho_t(x)} \left[ \norm{u_t(x)}^2 \right] dt
\end{aligned}
\end{equation}
As our benchmark problem, we investigate the dynamics of cells based on limited observations, focusing on the single-cell RNA sequencing data of embryoid bodies as analyzed by \citet{neklyudov2023action}. This dataset offers sparse observations in a 5-dimensional PCA decomposition of the original cell data introduced by \cite{moon2019visualizing} at discrete time points $t_0 = 0, t_1 = 1, t_2 = 2, t_3 = 3, t_4 =  4$. Our objective is twofold: first, to fit the time-continuous distribution of the dataset with given sparse observations, and second, to obtain optimal transport (OT) paths between these marginal distributions.

Numerous methods exist for learning continuous system dynamics from snapshots of temporal marginals. For example, the Neural SDE framework\citet{li2020scalable} —an extension of Neural ODEs\citet{chen2019neuralordinarydifferentialequations}—offers a robust approach to learning stochastic dynamics by employing scalable gradients computed via the adjoint sensitivity method. Similarly, OT-Flows\citet{onken2021ot} build upon Neural ODEs by integrating regularizations derived from optimal transport theory. In contrast, action matching\citet{neklyudov2023action} avoids back-propagation through stochastic or deterministic differential equations, thereby achieving significantly faster training. We trust this expanded discussion clarifies the rationale behind our comparisons and highlights the strengths of each approach.

To evaluate the performance of our model on this problem setup as compared to the existing methods, we compute the Wasserstein-2 distance between our fitted model $\rho_t$ and the data distribution $q_t$ at $t = 0,1,2,3,4$. The Wasserstein-2 distance is computed with the samples we directly sample from the model marginal distribution $\rho_t$ and the held-out test data from the dataset. Additionally, we compute the Wasserstein-2 distance between test marginals and model marginals by transporting samples from the data marginal $q_{t_i}$ to our estimated marginals at the next time value $\widehat{q}_{t_{i+1}}$ using the trained velocity field. We report the results in Table~\ref{table:results_OT} where our method surpass the existing methods in terms of fitting the time-continuous distribution of the dataset with given sparse observations. 

Compared to prior works, we not only learn the transport map and the marginal densities of the dataset, but also optimize the model for the kinetically optimal transport map. 
Our model has the flexibility in terms of training for the kinetic optimal transport map because of the additional learnable component $v_t^\theta$ that can be incorporated into the flux term (\Cref{sec:divfree}), all the while capable to be learned without sequential simulation of the underlying dynamical system. 
Results of optimizing for kinetically optimal transport map are reported in Table~\ref{table:results_kinetic_OT}. As compared to the factorized model, which is simpler and easier to train, the autoregressive model achieves better performance in both density fitting and optimizing for the optimal transport. 

\subsection{Mean-field stochastic optimal control}\label{sec:soc}

Stochastic optimal control (SOC; \citealt{mortensen1989stochastic, fleming2012deterministic, kappen2005path}) aims at finding the optimal dynamics model given an objective function, instead of data observations. SOC problems arise in wide variety of applications in sciences and engineering \citep{pham2009continuous,FLEMING2004979,zhang2022pathintegralsamplerstochastic,holdijk2023stochasticoptimalcontrolcollective,hartmann2013characterization,HartmannCarsten} and we provide numerical evidence to illustrate that our framework can be extended to solving SOC problems, including mean-field type of SOC problems~\citep{bensoussan2013mean}, which have wide applications in finance~\citep{FLEMING2004979,pham2009continuous} and robotics~\citep{theodorou2011iterative, pavlov2018narrow}. 
Reducing the SOC problem into our setting in \cref{eq:problem_formulation}, we have the following objective function:
\begin{equation}
\begin{aligned}
    L_\text{SOC} = & \int_0^1 \phi_t(\rho_t) dt + \E_{x\sim \rho_t} \left[ \frac{1}{2\sigma_t^2}\norm{u_t(x) - v_t(x)}^2 \right] dt \\ & + \Phi(\rho_1) + \E_{x_0 \sim q_0} \left[ -\log\rho_0(x_0) \right]
\end{aligned}
\end{equation}
where $q_0$ is a given initial distribution, $v_t$ is a given base drift function, and we use $\Phi(\rho_1) = \E_{x_1 \sim q_1} \left[ -\log\rho_1(x_1) \right]$ as the terminal cost so that the model can also be fitted to a given terminal distribution $q_1$. 

For our task, we formulate problems with circular obstacles that the model must navigate around. In particular, for circular obstacles with radius $R$ and center coordinate $c$, the running cost is defined as: 
\begin{equation}
\begin{aligned}
    \int_0^1 \phi_t(\rho_t)dt = & \E_{X_t \sim \rho_t} [\text{softplus} \left(R^2 - (X_t - c)^2\right)] \\ & + \eta \E_{X_t \sim \rho_t} [\log \rho_t (X_t) ] 
\end{aligned}
\end{equation}
where $\E_{X_t \sim \rho_t} [\log \rho_t (X_t) ]$ is the entropy of the model---\ie, a mean-field cost---used to encourage the model to find all the possible paths and $\eta$ is a weighting.

We test our method on the motion planning tasks introduced by \citet{le2023acceleratingmotionplanningoptimal}. The task is to navigate from the source to the target distribution while avoiding randomly initialized circular obstacles, where we use the entropy regularization to encourage finding multiple paths and to ensure we find robust solutions.
We visualize the trained model in Figure~\ref{fig:transport_mean_field}, where our framework trained with diffusion coefficient $g_t = 0$ can handle different environments and can also be used to produce reasonable samples when additional noise is present, \ie, $g_t > 0$. 

\begin{figure}[t]
    \centering
    \includegraphics[width=0.32\linewidth]{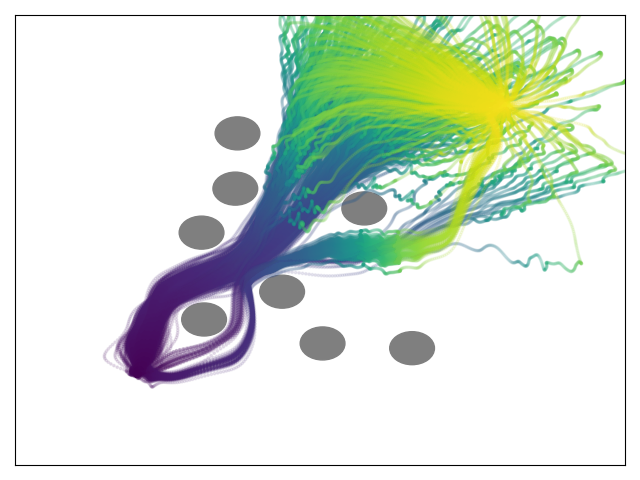}
    \includegraphics[width=0.32\linewidth]{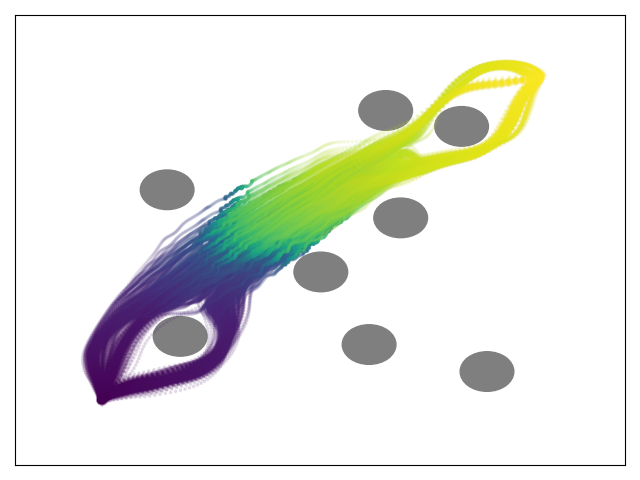}
    \includegraphics[width=0.32\linewidth]{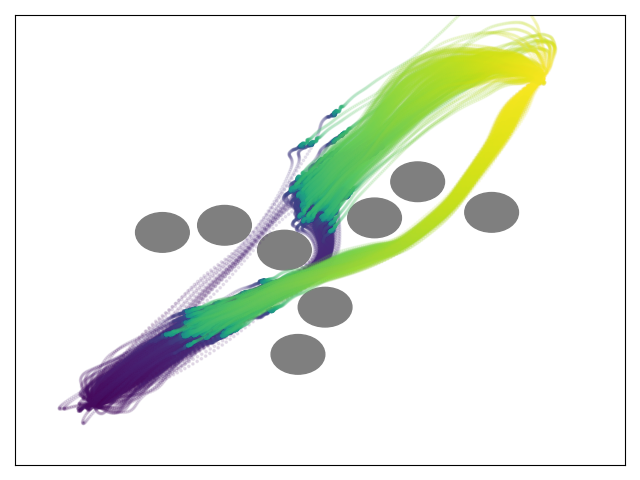}
    \includegraphics[width=0.32\linewidth]{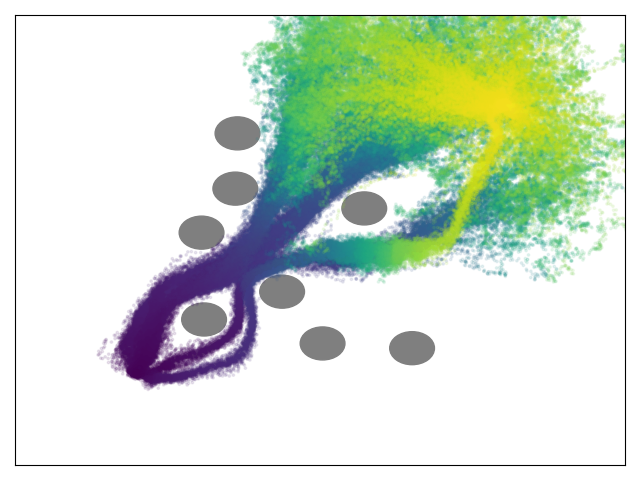}
    \includegraphics[width=0.32\linewidth]{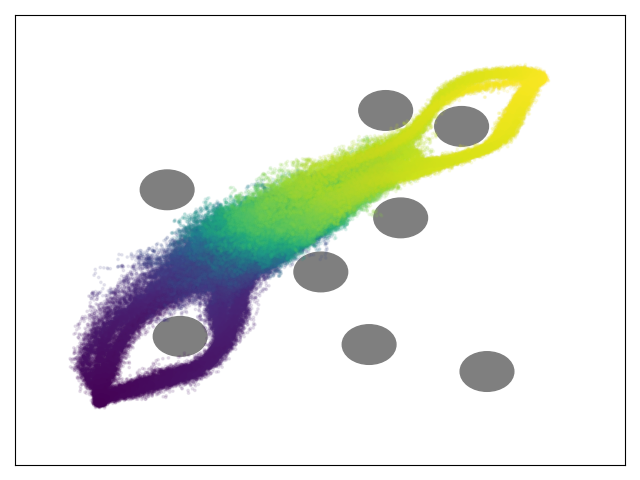}
    \includegraphics[width=0.32\linewidth]{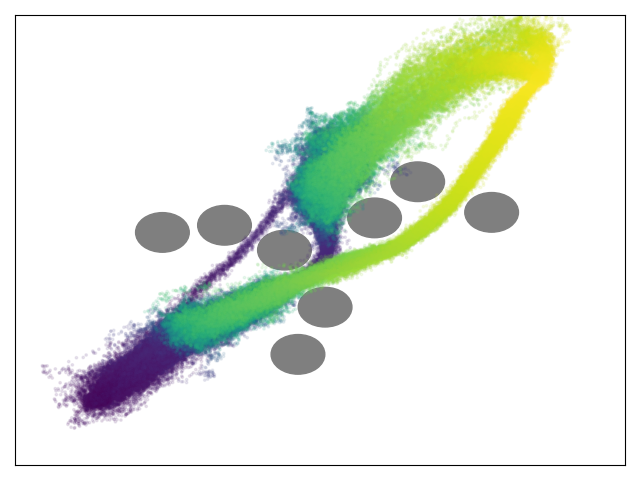}
    \caption{Transport paths of a trained factorized model on the motion planning task of different environments with randomly initialized circular obstacles. We train the model with a diffusion coefficient $g_t = 0.0$ and we sample the model via solving \cref{eq:sde} with $g_t = 0.0$ (first row) and $g_t = 0.5$ (second row). Note that in the case of $g_t = 0.0$, \cref{eq:sde} reduces to a deterministic ODE. }
    \label{fig:transport_mean_field}
\end{figure}

\section{Conclusion}

We have proposed a simulation-free framework for training continuous-time stochastic processes over a large range of objectives, by combining Neural Conservation Laws with likelihood-based models.
We demonstrated the flexibility and capacities of our method on various applications, including spatio-temporal generative modeling, learning optimal transport between arbritrary densities, and mean-field stochastic optimal control. 
Especially at low dimensional settings, our method easily outperforms existing methods. However, the reliance on likelihood-based models make it difficult to be scaled up to high dimensions. 
We acknowledge these limitations and leave them for future works. 

\section{Acknowledgements}

MH is supported on the Meta Platforms, Inc project entitled Probabilistic Deep Learning with Dynamical Systems. EVE is supported by the National Science Foundation under Awards DMR-1420073, DMS-2012510, and DMS-2134216, by the Simons Collaboration on Wave Turbulence, Grant No. 617006, and by a Vannevar Bush Faculty Fellowship.

% References
\bibliography{uai2025}

\newpage

\onecolumn
\appendix

\title{Simulation-Free Differential Dynamics \\through Neural Conservation Laws\\(Supplementary Material)}
\maketitle

\section{Jacobian symmetrization loss}

With the flexible learnable component $\ft$, we can learn the Hodge decomposition of the velocity field to remove the extra divergence-free (i.e., rotational) components of the model. As a direct consequence of the the Benamou-Brenier formula \cite{villani2021topics,albergo2022building}, the velocity field that achieves the optimal transport has no divergence-free component. 

A necessary and sufficient condition for a velocity field of being a gradient field is having symmetric Jacobian matrix with respect to all spatial dimensions. By the Hodge decomposition, any time-dependent velocity field $v^t : \R^{D+1} \rightarrow \R$ can be expressed as a sum of a divergence-free field and a gradient field:
\begin{equation}
    v^t = \nabla \phi^t + \eta^t 
\end{equation}
where $\eta^t$ is divergence-free. Let $J^t: \R^{D+1} \rightarrow \R^{D\times D}$ be the Jacobian matrix of $v^t$ with respect to its spatial dimensions and we denote its $(i,j)$ entry by $J_{i,j}^t = \partial_{x_j} v^{t}_i$. Then, the Jacobian matrix is symmetric if and only if $\partial_{x_j} v^t_{i} = \partial_{x_i} v^t_{j}$ for any $i,j$. It follows that $\partial_{x_j} v^t_{i} = \partial_{x_i} v^t_{j}$. Consequently, $\int_{-\infty} v^t_i d x_i = \int_{-\infty} v^t_j d x_j$ for any $i,j$. Then let $\phi^t = \int_{-\infty} v^t_1 d x_1$ and we obtain $v^t = \nabla \phi^t$ is a gradient field. Conversely, if $v^t = \nabla \phi^t$ is a gradient field, then $\partial_{x_j} v^t_{i} = \partial_i \partial_j \phi^t = \partial_{x_i} v^t_{j}$ and $J_{i,j}^t$ is symmetric. 

Motivated by this observation of the equivalence between the symmetry of the Jacobian and the OT plan, we train $\ft$ with the loss
\begin{equation}
    L_{\text{OT}} = \E_{t, x\sim \rho_t (x)} || J_t^\theta (x) - (J_t^\theta)^T (x) ||_F
\end{equation}
where $||\cdot||_F$ denotes the Frobenius norm. We can compute this loss by using the Hutchinson trace estimator \citet{hutchinson1989stochastic}. Let $v \sim \mathcal{N}(0,I)$ be a random Gaussian vector and $u_t (x) = v^T J_t^\theta (x)$ and $w_t(x) = J_t^\theta (x) v$. Computing $u$ and $w$ requires one vector-Jacobian product (VJP) and one Jacobian-vector product (JVP), respectively.  Therefore,
\begin{equation}
\begin{aligned}
    L_{\text{OT}} & = \E_{t, x\sim \rho_t (x)} || J_t^\theta (x) - (J_t^\theta)^T (x) ||_F \\
    & = \E_{v\sim \mathcal{N}(0,I);t, x\sim \rho_t (x) } [v^T \left(J_t^\theta (x) - (J_t^\theta)^T (x)\right) \left(J_t^\theta (x) - (J_t^\theta)^T (x)\right) v] \\ & = \E_{v\sim \mathcal{N}(0,I);t, x\sim \rho_t (x) } [u^T_t(x)w_t(x) - 2w^T_t(x)w_t(x) + w^T_t(x)u_t(x)]
\end{aligned}
\end{equation}
So, the stochastic estimation of the loss takes only one VJP and one JVP at each sample, which is computationally feasible even in high dimensions.

While this could arguably be better than regularizing kinetic energy for finding kinetic optimal solutions, as it no longer requires an explicit trade-off between kinetic energy and the other cost functions, we did not observe a meaningful improvement over simply regularizing kinetic energy.
\newpage
\section{Proof of Lemma~\ref{lemma:div_free} and Theorem~\ref{theorem:no_flux}}
\label{app:proof_1}
\paragraph{Lemma~\ref{lemma:div_free}}
\begin{proof}
Given the definition of $b_t^\theta$ provided in \cref{eq:autoregressive_bt}, we directly compute $\partial_{x_i} \bt$ as follows:
\begin{equation}
\partial_{x_i} [\bt]_{i}(x) =
\begin{cases}
    & \sigma'(x_{i})\partial_t \left( \prod_{j=1}^{D-1} f^\theta_t(x_{j}| x_{1:j-1}) \right),  \quad \text{if } i = D\\
    & - \left( \prod_{j=2}^D \sigma'(x_{j}) \right) \partial_t  f^\theta_t(x_{1}),  \quad \text{if } i = 1\\
    & \left( \prod_{j=i+1}^D \sigma'(x_{{j}}) \right) \left(\sigma'(x_{i}) -  f^\theta_t(x_{i} | x_{1:i-1})\right)  \partial_t \left( \prod_{j=1}^{i-1} f^\theta_t(x_{j} | x_{1:j-1})\right) \\ & -\left( \prod_{j=i+1}^D\sigma'(x_{{j}}) \right)  \partial_t  f^\theta_t(x_{i} | x_{1:i-1})  \left( \prod_{j=1}^{i-1} f^\theta_t(x_{j} | x_{1:j-1} )\right), \quad  \text{otherwise}
\end{cases}
\end{equation}
Note that for the cases in which $i \in \{2,\dotsm, D-1\}$, we can further simplify $\partial_{x_i} [\bt]_{i}(x)$ as follows
\begin{equation}
    \partial_{x_i} \bt = \left(\prod_{j=i}^D\sigma'(x_{{j}})\right) \partial_t \left( \prod_{j=1}^{i-1} f^\theta_t(x_{j} | x_{1:j-1})\right) - \left(\prod_{j=i+1}^D\sigma'(x_{{j}}) \right)\partial_t \left( \prod_{j=1}^{i} f^\theta_t(x_{j} | x_{1:j-1})\right)
\end{equation}
Let
\begin{equation}
    m_i = \left(\prod_{j=i+1}^D\sigma'(x_{{j}}) \right)\partial_t \left( \prod_{j=1}^{i} f^\theta_t(x_{j} | x_{1:j-1})\right)
\end{equation}
Then,
\begin{equation}
    \partial_{x_i} [\bt]_{i}(x) = m_{i-1} - m_{i}
\end{equation}
with $\partial_{x_1} [\bt]_{1}(x)= -m_1$ and $\partial_{x_D} [\bt]_{D}(x)= m_{D-1}$.
Therefore, it follows immediately that 
\begin{equation}
    \nabla \cdot \bt (x) = \sum_{i=1}^D \partial_{x_i} [\bt]_{i}(x) = -m_1 + \sum_{i=2}^{D-1} m_{i-1} - m_{i} + m_{D-1} = 0
\end{equation}
\end{proof}

\paragraph{Theorem~\ref{theorem:no_flux}}
\begin{proof}
    By Lemma~\ref{lemma:div_free} and Lemma~\ref{lemma:continuity}, $\at$ constructed by \cref{eq:autoregressive_at} and $\bt$ constructed by~\cref{eq:autoregressive_bt} satisfies the continuity \cref{eq:continuity_eq}. Now, we compute the normalization constant of the resulting density $\rho_t^\theta = \nabla \cdot a_t^\theta$ as follows
    \begin{equation}
        \int_{\R^D} \rho_t^\theta dx = \int_{\R^D} \nabla \cdot a_t^\theta dx = \int_{\R^D}  \prod_{i=1}^D f_t^\theta (x_i|x_{1:i-1}) dx = 1
    \end{equation}

Therefore, the density $\rho_t^\theta$ is properly normalized. Now, we check for the spurious fluxes for all dimensions by investigating the following limit:
\begin{equation}
    \lim_{x \rightarrow \infty} \left|j^\theta_t\right| (x) = \lim_{x \rightarrow \infty} \left|\partial_t \at - \bt \right| \leq \sum_{i = 1} \lim_{x \rightarrow \infty} \left|\partial_t [\at]_i - [\bt]_i \right|
\end{equation}
Hence, it suffices to prove 
\begin{equation}
    \lim_{x \rightarrow \infty} \left|\partial_t [\at]_i - [\bt]_i \right| = 0
\end{equation}
for all $i \in \{1,\dotsm,D\}$ to show that there is no spurious flux in our model. 

For $i = D$, by \cref{eq:spurious_flux_autoregressive} and \cref{eq:autoregressive_bt}, we have 
\begin{equation}
\begin{aligned}
    \lim_{x \rightarrow \infty} \left|\partial_t [\at]_i - [\bt]_i \right| (x) & = \lim_{x_{1:D-1} \rightarrow \infty} \lim_{x_D \rightarrow \infty}  \left|\partial_t [\at]_i - [\bt]_i \right| (x) \\ & = 
    \lim_{x_{1:D-1} \rightarrow \infty} \lim_{x_D \rightarrow \infty} | - \partial_t \left( \prod_{i=1}^{D-1} f_t^\theta(x_{i}| x_{1:i-1}) \right) + \sigma(x_{D})\partial_t \left( \prod_{j=1}^{D-1} f^\theta_t(x_{j}| x_{1:j-1}) \right) | \\ & = \lim_{x_{1:D-1} \rightarrow \infty} \lim_{x_D \rightarrow \infty} | \partial_t \left( \prod_{i=1}^{D-1} f_t^\theta(x_{i}| x_{1:i-1}) \right) \left(\sigma (x_D) - 1\right) | \\ & = 
    \lim_{x_{1:D-1} \rightarrow \infty} |\partial_t \left( \prod_{i=1}^{D-1} f_t^\theta(x_{i}| x_{1:i-1}) \right) |\lim_{x_D \rightarrow \infty} |\left(\sigma (x_D) - 1\right) | \\ & = 0
\end{aligned} 
\end{equation}

For $i = 1$, since $\lim_{x \rightarrow \infty} F^\theta_t(x_{1}) = 1$ for all $t \geq 0$, $\lim_{x \rightarrow \infty} \partial_t F^\theta_t(x_{1}) = 0$. Therefore, for $i = 1$, 
\begin{equation}
\begin{aligned}
    \lim_{x \rightarrow \infty} \left|\partial_t [\at]_i - [\bt]_i \right| = \lim_{x \rightarrow \infty}  - \left( \prod_{j=2}^D \sigma'(x_{j}) \right) \partial_t  F^\theta_t(x_{1}) = 0
\end{aligned}
\end{equation}

For $i \in \{2,\dotsm, D-1\}$, 
\begin{equation}
\begin{aligned}
    \lim_{x \rightarrow \infty} \left|\partial_t [\at]_i - [\bt]_i \right| = \left(\prod_{j=i+1}^D\sigma'(x_{{j}}) \right)\partial_t \left( \prod_{j=1}^{i-1} f^\theta_t(x_{j} | x_{1:j-1})\right) \left(\sigma (x_i) - \partial_t F^\theta_t(x_{i} | x_{1:i-1})\right) = 0
\end{aligned}
\end{equation}

Hence, we deduce that there is no spurious flux given our construction of $\at$ and $\bt$. 

\end{proof}

\section{Proof of Proposition~\ref{prop:mixture}}

\label{app:prop1_proof}

\begin{proof} We check that $\rho_t$ and $u_t$ satisfy \cref{eq:fp_eq}:
\begin{equation}
\begin{split}
    \partial_t \rho_t & = \sum_{k=1}^K \gamma^k \left( \partial_t \rho_t^k \right) 
    = \sum_{k=1}^K \gamma^k 
    \left( - \nabla \cdot (u_t^k \rho_t^k) + \tfrac12\diff^2 \Delta \rho_t^k \right) \\
   &  = - \nabla \cdot \sum_{k=1}^K \gamma^k (u_t^k \rho_t^k) + \tfrac12 \diff^2 \Delta \sum_{k=1}^K \gamma^k \rho_t^k  \\
    &= - \nabla \cdot \left( \sum_{k=1}^K \frac{\gamma^k \rho_t^k}{\rho_t} u_t^k \right) \rho_t + \tfrac12 \diff^2 \Delta \sum_{k=1}^K \gamma^k \rho_t^k 
    = - \nabla \cdot u_t \rho_t + \tfrac12 \diff^2 \Delta \rho_t
\end{split}
\end{equation}
\end{proof}

\section{Experimental Setup}\label{app:experimental_setup}

\paragraph{Neural Network Architecture}
For training the autoregressive model, we use the MADE architecture \citep{germain2015made} with sinusoidal time embeddings of width $128$ \citep{tancik2020fourfeat}. For the neural networks we use to parameterize the mean and the scale of both the autoregressive model and the factorized model, we pass the input first into the sinusoidal time embeddings before feeding into a four-layer MLP of hidden dimension $256$ on each layer. 
\paragraph{Training Details}
For all the numerical experiments we present, we use a learning rate of $3e-4$ with the Adam optimizer \citep{kingma2014adam} and a cosine annealing learning rate scheduler. 
\paragraph{Spatio-temporal Generative Modeling}
The total number of iterations we run for the experiments are generally $10^3$ epochs with a batch size of $256$. We found
that the training is stable with a simple four-layer MLP parametrization for the mean and the scale of the mixtures of factorized logistics. Also, the MLP parameterization along with the mixture combinations in the factorzied model turned out to be expressive enough for the experiments we have explored. 
\paragraph{Learning To Transport With Optimality Conditions}

For the single-cell RNA sequence dataset used in \citep{moon2019visualizing}, we find that both the factorized model and the autoregressive model will easily get overfitted if we use more than $64$ modes in the mixture. For the numerical results we are reporting, we use mixtures of size $L=16$ for each of the coordinates in the autoregressive model (\eqref{eq:autoregressive_mixture}), and we use a mixture of size $K=32$ for the factorized model (\eqref{eq:mixture_rho}). Also, we find that having the term $\int_{t_0}^{t_n} \E_{x \sim \rho_t(x)} \left[ \norm{u_t(x)}^2 \right] dt$ in the loss objective is extremely helpful for both finding the kinetic optimal paths and preventing overfitting. 

\paragraph{Mean-field Stochastic Optimal Control} 

To achieve consistent results for this experiment, we train the objective function $L_{\text{SOC}}$ by gradually introducing different terms in it. We first train the log-likelihood term $\E_{x_0\sim q_0} [-\log \rho_0 (x_0)] + \E_{x_1\sim q_1} [-\log \rho_1 (x_1)]$ for $10^3$ iterations with a batch size of $512$. Then, we introduce the term $\E_{x\sim \rho_t} \left[ \frac{1}{2\sigma_t^2}\norm{u_t(x) - v_t(x)}^2 \right] dt$ for another $10^3$ iterations. Finally, we introduce the running cost $\int_{0}^{1} \phi_t (\rho_t) dt$ and train for $2\times 10^4$ iterations. This training technique helps stabilize the training.

\end{document}